\documentclass[journal]{IEEEtran}

\usepackage[hyphens]{url}  
\usepackage{graphicx}

\usepackage{setspace}
\usepackage{amsmath}
\usepackage{amsthm}
\usepackage{amssymb}
\usepackage{flowchart}
\usepackage{tikz}
\usepackage{graphicx}
\usepackage{subcaption}
\usepackage{mathtools}
\usepackage{multirow}
\usepackage{bm}
\usepackage{balance}
\usepackage{makecell}
\usepackage{algorithm}
\usepackage{algorithmic}
\usepackage{threeparttable}

\newtheorem{theorem}{Theorem}

\newtheorem{corollary}{Corollary}

\newcommand{\w}{\bm{\mathrm w}}

\title{A Multi-Task Gradient Descent Method for Multi-Label Learning}

\author{Lu Bai\textsuperscript{\rm 1}, Yew-Soon Ong\textsuperscript{\rm 1}, Tiantian He\textsuperscript{\rm 1}, Abhishek Gupta\textsuperscript{\rm 2}\\
\textsuperscript{\rm 1}{School of Computer Science and Engineering, Nanyang Technological University, Singapore}\\
\textsuperscript{\rm 2}{Singapore Institute of Manufacturing Technology (SIMTech), A*STAR, Singapore}}

\begin{document}

\maketitle

\begin{abstract}
Multi-label learning studies the problem where an instance is associated with a set of labels. By treating single-label learning problem as one task, the multi-label learning problem can be casted as solving multiple related tasks simultaneously. In this paper, we propose a novel Multi-task Gradient Descent (MGD) algorithm to solve a group of related tasks simultaneously. In the proposed algorithm, each task minimizes its individual cost function using reformative gradient descent, where the relations among the tasks are facilitated through effectively transferring model parameter values across multiple tasks. Theoretical analysis shows that the proposed algorithm is convergent with a proper transfer mechanism. Compared with the existing approaches, MGD is easy to implement, has less requirement on the training model, can achieve seamless asymmetric transformation such that negative transfer is mitigated, and can benefit from parallel computing when the number of tasks is large. The competitive experimental results on multi-label learning datasets validate the effectiveness of the proposed algorithm.
\end{abstract}

\section{Introduction}

Multi-label learning deals with the problem that one instance is associated with multiple labels, such as a news document can be labeled as sports, Olympics, and ticket sales \cite{zhang2014review}. Formally, let $\mathcal X\in\mathbb R^p$ denote the $p$-dimensional feature space and $\mathcal Y\in\mathbb R^{T}$ denote the label space with $T$ class labels. Given the multi-label training set $\mathcal D=\{(\bm x_j, \bm y_j)| 1\leq j\leq n\}$, where $n$ is number of instances, $\bm x_j\in\mathcal X$ is the feature vector for the $j$-th instance and $\bm y_j\in\{0,1\}^{T}$ is the set of labels associated with the $j$-th instance. The task of multi-label learning is to learn a function $h:\mathcal X\to\mathcal Y$ from $\mathcal D$ which can assign a set of proper labels to an instance. 

One straightforward method to solve the multi-label learning problem is to decompose the problem into a set of independent binary classification problems \cite{boutell2004learning}. This strategy is easy to implement and existing single-label classification approaches, e.g., logistic regression and SVM, can be utilized directly. However, as can be seen by the news document example, an instance with the Olympics label has a high probability to have the label of sports. The correlations among the labels may provide useful information for one another and help to improve the performance of multi-label learning \cite{zhang2014review,gibaja2015tutorial}. 

Over the past years, a lot of methods have been proposed to improve the performance of multi-label learning by exploring the label correlations. Methods such as classifier chains \cite{read2011classifier}, calibrated label ranking \cite{furnkranz2008multilabel}, and random $k$-labelsets \cite{tsoumakas2010random} usually have high complexity with a large number of class labels. \cite{cai2013new,huang2016learning,huang2018joint} considered taking the label correlations as prior knowledge and incorporating it into the model training to utilize the label correlations. \cite{huang2012multi,zhu2017multi} exploited label correlations through learning a latent label representation and optimizing label manifolds. \cite{feng2019collaboration} explored the correlations by solving an optimization problem which models the contribution of related labels, and then incorporating the learned correlations into the model training. In the existing approaches, a well-designed training model is required to achieve notable performances.

Inspired by the merits of first-order methods and taking into account the importance of correlations among the labels, a novel Multi-task Gradient Descent (MGD) algorithm is proposed in this paper to solve the multi-label learning problem. Treating a single-label learning problem as single task, the multi-label learning can be casted as solving multiple related tasks simultaneously. In MGD, each task minimizes its individual cost function using the gradient descent algorithm and the similarities among the tasks are then facilitated through transferring model parameter values during the optimization process of each task. We prove the convergence of MGD when the transfer mechanism and the step size of gradient descent satisfy certain easily achievable conditions. Compared with the existing approaches, MGD is easy to implement, has less requirement on the training model, and can achieve seamless asymmetric transformation such that negative transfer is mitigated \cite{lee2016asymmetric}. In addition, MGD can also  benefit from parallel computing with small amount of information processed centrally when the number of tasks is large.

The rest of the paper is organized as the follows. Previous works related to multi-label learning are firstly reviewed. Secondly, we introduce the proposed MGD and provide the theoretical analysis, including model convergence and computational complexity. Thirdly, we present how the proposed MGD is extensively tested on real multi-label learning datasets and compared with strong baselines. At last, we summarize the proposed approach and the contributions of the paper.

\section{Related Work}
Based on the order of information being considered, existing multi-label learning approaches can be roughly categorized into three major types \cite{zhang2014review}. For first-order methods, the label correlations are ignored and the multi-label learning problem is handled in a label by label manner, such as BR \cite{boutell2004learning} and LIFT \cite{zhang2014lift}. Second-order methods consider pairwise relations between labels, such as LLSF \cite{huang2016learning} and JFSC \cite{huang2018joint}. High-order methods, where high-order relations among label subsets or all the labels are considered, such as RAkEL \cite{tsoumakas2010random}, ECC \cite{read2011classifier}, LLSF-DL \cite{huang2016learning}, and CAMEL \cite{feng2019collaboration}. Generally, the higher the order of correlations being considered, the stronger is the correlation-modeling capabilities, while on the other hand, the more computationally demanding and less scalable the approach becomes.

Treating a single-label learning problem as one task, the multi-label learning problem can be seen as a special case of multi-task learning problem, where the feature vectors $\bm x_j$ for $j=1,...,n$ are the same for different tasks. In majority of multi-task learning method, the relations among the tasks are promoted through regularization in the overall objective function that composed of all the tasks' parameters, such as feature based approaches \cite{argyriou2007multi,obozinski2006multi,liu2009blockwise,han2014encoding,chen2009convex} and task relation based approaches \cite{evgeniou2004regularized,gornitz2011hierarchical,zhang2014regularization}.
Specifically, for the second-order multi-label learning approaches in \cite{cai2013new,huang2016learning,huang2018joint}, the label correlation matrix, which is taken as a prior knowledge obtained based on the similarity between label vectors, is often incorporated as a structured norm regularization term that regulates the learning hypotheses or perform label-specific feature selection and model training.

In contrast to the existing multi-label and multi-task learning approaches which incorporate correlation information into the model training process in the form of regularization, MGD serves as the first attempt to incorporate the correlations by transferring model parameter values during the optimization process of each task, i.e., when minimizing its individual cost function.

\section{The MGD Approach}
In this section, we elaborate the proposed MGD algorithm for multi-label learning. We firstly introduce the mathematical notations used in the manuscript. We then generically formulate the multi-label learning problem and introduce how MGD can effectively solve multi-label learning problem via the reformative gradient descent where the correlated parameters are transferred across multiple tasks. At last, we perform the theoretical analysis of MGD, including convergence proof and computational complexity.

Throughout this paper, normal font small letters denote scalars, boldface small letters denote column vectors, and capital letters denote matrices. $\bm 0$ denotes zero column vector with proper dimension, $I_n$ denotes identity matrix of size $n\times n$. $A'$ denotes the transpose of matrix $A$ and $\otimes$ denotes the Kronecker product. $[\bm z_i]_{\text{vec}}$ denotes a concatenated column vector formed by stacking $\bm z_i$ on top of each other, and $\text{diag}\{z_i\}$ denotes a diagonal matrix with the $i$-th diagonal element being $z_i$. The norm $\|\cdot\|$ without specifying the subscript represents the Euclidean norm by default. Following the notations used in Introduction, we alternatively represent the training set as $\mathcal D = \{(X,Y)\}$ where $X = [\bm x_1,...,\bm x_n]'\in\mathbb R^{n\times d}$ denotes the instance matrix and $Y = [\bm y_1,...,\bm y_n]'\in\mathbb R^{n\times T}$ denotes the label matrix. In addition, we denote the training set for label $i\in\{1,...,T\}$ as $\mathcal D_i = \{(X, \bm y^i)\}$ where $\bm y^i\in\mathbb R^n$ is the $i$-th column vector of the label matrix $Y$. 

\subsection{Problem Formulation}
Treating each single-label learning problem as one task, we have $T$ tasks to be solved simultaneously. Each task $i\in\{1,...,T\}$ aims to minimize its own cost function 
\begin{align}\label{prob1}
\min_{\bm w_i} \ f_i(\bm w_i),
\end{align}
where $\bm w_i\in\mathbb R^{d}$ is the model parameter and $f_i: \mathbb R^{d}\to\mathbb R$ is the cost function of the $i$-th task with training dataset $\mathcal D_i$. In this paper, we do not restrict the specific form of the cost functions. In particular, the cost functions $f_i(\bm w_i)$ is assumed to be strongly convex, twice differentiable, and the gradient of $f_i$ is Lipschitz continuous with constant $L_{f_i}$, i.e.,
\begin{align*}
\|\nabla f_i(\bm u)-\nabla f_i(\bm v)\|\leq L_{f_i} \|\bm u-\bm v\|,\quad  \forall \bm u, \bm v\in\mathbb R^{d}.
\end{align*}
Cost functions such as mean squared error with norm 2 regularization and cross-entropy with norm 2 regularization apply. Non-differentiable cost functions where norm 1 regularization is used can also be approximated considered \cite{schmidt2007fast}. 
Since $f_i(\bm w_i)$ is strongly convex and twice differentiable, there exists positive constant $\xi_i$ such that $\nabla^2 f_i(\bm u)\geq \xi_i I_d$. As a result, we have
\begin{align*}
\xi_iI_d\leq\nabla^2f_i(\bm u)\leq L_{f_i}I_d, \ \forall \bm u\in\mathbb R^d.
\end{align*}

\subsection{The Proposed Framework}
Equation~\eqref{prob1} is solved using the gradient descent iteration,
\begin{align}\label{gd}
\bm w_i^{t+1} = \bm w_i^{t}-\alpha\nabla f_i(\bm w_i^t),
\end{align}
where $t$ is the iteration index, $\alpha$ is the step size, and $\nabla f_i(\bm w_i^t)\in \mathbb R^d$ is the gradient of $f_i$ at $\bm w_i^t$. As there are relations among tasks, we are able to improve the learning performance by considering the correlation of parameters belonging to different tasks. Based on this idea, we propose a reformative gradient descent iteration, which allows the values of the model parameters during each iteration to be transferred across similar tasks. The MGD is designed as follows, 
\begin{align}\label{mgd}
\bm w_i^{t+1} = \sum_{j=1}^T m_{ij}^t\bm w_j^t-\alpha\nabla f_i(\bm w_i^t),\ i=1,...,T,
\end{align}
where $m_{ij}^t$ is the transfer coefficient describes the information flow from task $j$ to task $i$, which satisfies the following conditions,
\begin{subequations}\label{mij}
\begin{align}
m_{ij}^t\geq 0, \label{m1}\\
\sum_{j=1}^T m_{ij}^t = 1. \label{m3}
\end{align}
\end{subequations}

From \eqref{m3}, we have $m_{ii}^t = 1-\sum_{j\neq i}m_{ij}^t$. Rewriting iteration \eqref{mgd} as follows
\begin{align*}
\bm w_i^{t+1} = & m_{ii}^t\bm w_i^t+\sum_{j\neq i}m_{ij}^t\bm w_j^t-\alpha\nabla f_i(\bm w_i^t)\\
=&(1-\sum_{j\neq i}m_{ij}^t)\bm w_i^t+\sum_{j\neq i}m_{ij}^t\bm w_j^t-\alpha\nabla f_i(\bm w_i^t).
\end{align*}
$m_{ij}^t$ can be rescaled as 
\begin{align}\label{mbar}
\bar m_{ij}^t=
\begin{cases}
\frac{1}{\alpha\sigma}m_{ij}^t, & j\neq i,\\
1-\frac{1}{\alpha\sigma}\sum_{j\neq i}m_{ij}^t, & j=i,
\end{cases}
\end{align}
where $\sigma$ is a positive constant and satisfies the condition 
\begin{align}\label{sigma}
1-\frac{1}{\alpha\sigma}\sum_{j\neq i}m_{ij}^t>0.
\end{align}
Given \eqref{mbar}, $m_{ij}^t$ is parameterized by $\sigma$.
With the rescaling, the iteration in \eqref{mgd} can be alternatively expressed as
\begin{align}\label{barmite}
\nonumber \bm w_i^{t+1}\! &= \! (1-\!\alpha\sigma\!\sum_{j\neq i}\bar m_{ij}^t)\bm w_i^t+\alpha\sigma\sum_{j\neq i}\bar m_{ij}^t\bm w_j^t-\alpha\nabla f_i(\bm w_i^t)\\
\nonumber = &\bm w_i^t-\alpha\sigma(1-\bar m_{ii}^t)\bm w_i^t+\alpha\sigma\sum_{j\neq i}\bar m_{ij}^t\bm w_j^t-\alpha\nabla f_i(\bm w_i^t)\\
= & (1-\alpha\sigma)\bm w_i^t+\alpha\sigma\sum_{j=1}^T \bar m_{ij}^t\bm w_j^t-\alpha\nabla f_i(\bm w_i^t).
\end{align}

\subsection{Convergence Analysis}
In this section, we give the convergence property of the proposed MGD iteration based on the expression in \eqref{barmite}.

Denote $\bm w_i^*$ as the best coefficient of label predictor for task $i$,
$\tilde {\bm w_i}^t = \bm w_i^t-\bm w_i^*$, and $\bar L_{f_i} = \max_i\{L_{f_i}\}$. The following theorem gives the convergence property of the iteration \eqref{barmite} under certain conditions on the step-size parameter $\alpha$.
\begin{theorem}\label{theo1}
Under the iteration in \eqref{barmite} with the transfer coefficient $\bar m_{ij}^t$ satisfies 
\begin{align*}
&\sum_{j=1}^T\bar m_{ij}^t=1,\ \forall i,\\
&\bar m_{ij}^t\geq 0,\ \forall i,j,
\end{align*}
$\bm w_i^t$ is convergent if the step size $\alpha$ is chosen to satisfy
\begin{align}\label{alpha}
 0<\alpha<\frac{2}{2\sigma+ \bar L_{f_i}}.
 \end{align}
 Specifically,
 \begin{align}\label{steadystate}
\lim_{t\to\infty} \max_i\|\tilde{\bm w_i}^{t}\| 
\leq \frac{2\alpha\sigma\max_i\|\bm w_i^*\|+\alpha\max_i\|\nabla f_i(\bm w_i^*)\|}{1-(\bar \gamma+\alpha\sigma)},
\end{align}
where $\bar \gamma = \max_i\{|1-\alpha\sigma-\alpha \xi_i|,|1-\alpha\sigma-\alpha L_{f_i}|\}$.
\end{theorem}
\begin{proof}
Let the $i,j$-th element of $ \bar M^t\in\mathbb R^{T\times T}$ at iteration time $t$ being $\bar m_{ij}^t$, denote $\bar {\mathcal M}^t = \bar M^t\otimes I_d\in\mathbb R^{dT\times dT}$, $\w = [\bm w_1^{'}, ..., \bm w_T^{'}]'\in\mathbb R^{dT}$, and $\nabla f(\w^t) = [\nabla f_1(\bm w_1^t)',...,\nabla f_T(\bm w_T^t)']'\in\mathbb R^{dT}$. Note that we are using the typeface $\w$ to distinguish this from the single vector-valued variable $\bm w_i$. Write \eqref{barmite} into a concatenated form gives
\begin{align}\label{cmgd}
\w^{t+1} = (1-\alpha\sigma)\w^t+\alpha\sigma\bar {\mathcal M}^t\w^t-\alpha\nabla f(\w^t).
\end{align}

Denote $\w^* = [\bm w_1^{*'},...,\bm w_T^{*'}]'$ and $\tilde {\w}^t=\w^t-\w^*$. Subtracting $\w^*$ from both sides of \eqref{cmgd} gives
\begin{align}\label{ite}
\nonumber\tilde {\w}^{t+1} = & ((1-\alpha\sigma)I_{dT}+\alpha\sigma\bar{\mathcal M}^t)\w^t-\w^*-\alpha\nabla f(\w^t)\\
\nonumber= & ((1-\alpha\sigma) I_{dT}+\alpha\sigma\bar{\mathcal M}^t)\tilde {\w}^t-\alpha(\nabla f(\w^t)-\nabla f(\w^*))\\
\nonumber& +\alpha(\sigma(\bar{\mathcal M}^t-I_{dT})\w^*-\nabla f(\w^*))\\
\nonumber= & ((1-\alpha\sigma) I_{dT}+\alpha\sigma\bar{\mathcal M}^t)\tilde {\w}^t\\
\nonumber& -\alpha\int_0^1 \nabla^2 f(\w^*+\mu(\w^t-\w^*))d\mu \tilde{\w}\\
\nonumber& +\alpha(\sigma(\bar{\mathcal M}^t-I_{dT})\w^*-\nabla f(\w^*))\\
\nonumber= & ((1-\alpha\sigma) I_{dT}+\alpha\sigma\bar{\mathcal M}^t-\alpha H^t)\tilde {\w}^t\\
& +\alpha(\sigma(\bar{\mathcal M}^t-I_{dT})\w^*-\nabla f(\w^*)),
\end{align}
where $H^t = \int_0^1\nabla^2 f(\w^*+\mu(\w^t-\w^*))d\mu\in\mathbb R^{dT\times dT}$. It can be verified that $H^t$ is a block diagonal matrix and the block diagonal elements $H_i^t = \int_0^1\nabla^2 f_i(\bm w_i^*+\mu(\bm w_i^t-\bm w_i^*))d\mu\in\mathbb R^{d\times d}$ for $i=1,...,T$ are Hermitian. We use the block maximum norm defined in \cite{sayed2014diffusion} to show the convergence of the above iteration. The block maximum norm of a vector $\bm x = [\bm x_i]_{\text{vec}}\in\mathbb R^{dT}$ with $\bm x_i\in\mathbb R^d$ is defined as \cite{sayed2014diffusion}
\begin{align*}
\|\bm x\|_{b,\infty} = \max_{i}\|\bm x_i\|.
\end{align*}
The induced matrix block maximum norm is therefore defined as \cite{sayed2014diffusion}
\begin{align*}
\|A\|_{b,\infty} = \max_{\bm x\neq 0}\frac{\|A\bm x\|_{b,\infty}}{\|\bm x\|_{b,\infty}}.
\end{align*}
From the iteration in \eqref{ite} we have
\begin{align*}
&\|\tilde {\w}^{t+1}\|_{b,\infty} \leq  \|((1-\alpha\sigma) I_{dT}+\alpha\sigma\bar{\mathcal M}^t-\alpha H^t)\tilde {\w}^t\|_{b,\infty}\\
& +\alpha\|\sigma(\bar{\mathcal M}^t-I_{dT})\w^*-\nabla f(\w^*)\|_{b,\infty}\\
\leq & \|(1-\alpha\sigma) I_{dT}+\alpha\sigma\bar{\mathcal M}^t-\alpha H^t\|_{b,\infty}\|\tilde {\w}^t\|_{b,\infty}\\
& +\alpha\|\sigma(\bar{\mathcal M}^t-I_{dT})\w^*-\nabla f(\w^*)\|_{b,\infty}\\
\leq & (\|(1-\alpha\sigma) I_{dT}-\alpha H^t\|_{b,\infty}+\alpha\sigma\|\bar{\mathcal M}^t\|_{b,\infty})\|\tilde {\w}^t\|_{b,\infty}\\
& +\alpha\|\sigma(\bar{\mathcal M}^t-I_{dT})\w^*-\nabla f(\w^*)\|_{b,\infty}.
\end{align*}
From Lemma~D.3 in \cite{sayed2014diffusion}, we have
\begin{align*}
\|\bar{\mathcal M}^t\|_{b,\infty} = \|\bar M^t\|_{\infty} = 1,
\end{align*}
where the last equality comes from the fact that $\bar m_{ij}^t\geq0$ and the row summation of $\bar M^t$ is one. Since $\xi_iI_d\leq\nabla^2f_i(\bm w_i)\leq L_{f_i}I_d$, $\xi_iI_d\leq\int_0^1\nabla^2 f_i(\bm w_i^*+\mu(\bm w_i-\bm w_i^*)d\mu\leq L_{f_i}I_d$. Thus, $\|(1-\alpha\sigma) I_d-\alpha H_i^t\|\leq \gamma_i$ where $\gamma_i = \max\{|1-\alpha\sigma-\alpha \xi_i|,|1-\alpha\sigma-\alpha L_{f_i}|\}$. By the definition of induced matrix block maximum norm, we have
\begin{align*}
&\|(1-\alpha\sigma)I_{dT}-\alpha H^t\|_{b,\infty} \\
= &\max_{\bm x\neq 0}\frac{\|((1-\alpha\sigma)I_{dT}-\alpha H^t)\bm x\|_{b,\infty}}{\|\bm x\|_{b,\infty}}\\
\leq &\max_{\bm x\neq 0}\frac{\max_i\|((1-\alpha\sigma)I_{d}-\alpha H_i^t)\| \|\bm x\|_{b,\infty}}{\|\bm x\|_{b,\infty}}\\
=& \max_i \|(1-\alpha\sigma)I_{d}-\alpha H_i^t\|\\
\leq& \bar\gamma,
\end{align*}

where $\bar \gamma = \max\{\gamma_i\}$. Thus,
\begin{align}\label{proite1}
\nonumber\|\tilde{\w}^{t+1}\|_{b,\infty}\leq& (\bar \gamma+\alpha\sigma)\|\tilde{\w}^t\|_{b,\infty}\\
\nonumber &+\alpha\|\sigma(\bar{\mathcal M}^t-I_{dT})\w^*-\nabla f(\w^*)\|_{b,\infty}\\
\nonumber\leq & (\bar \gamma+\alpha\sigma)\|\tilde{\w}^t\|_{b,\infty}\\
&+2\alpha\sigma\|\w^*\|_{b,\infty}+\alpha\|\nabla f(\w^*)\|_{b,\infty}.
\end{align}
By choosing the step size $\alpha$ to satisfy $\bar\gamma+\alpha\sigma<1$, the iteration asymptotically converges. To ensure $\bar\gamma+\alpha\sigma<1$, it is sufficient to ensure
\begin{align*}
|1-\alpha\sigma-\alpha\xi_i|+\alpha\sigma<1 \ \text{and} \\
|1-\alpha\sigma-\alpha L_{f_i}|+\alpha\sigma<1,\ \forall i,
\end{align*}
which leads to
\begin{align*}
0<\alpha<\frac{2}{2\sigma+ \bar L_{f_i}}.
\end{align*}
From the iteration in \eqref{proite1}, we have
\begin{align*}
&\|\tilde {\w}^{t+1}\|_{b,\infty}\leq(\bar \gamma+\alpha\sigma)^{t+1}\|\tilde {\w}^{0}\|_{b,\infty}\\
&+(2\alpha\sigma\|\w^*\|_{b,\infty}+\alpha\|\nabla f(\w^*)\|_{b,\infty})\sum_{k=0}^t(\bar \gamma+\alpha\sigma)^k.
\end{align*}
Under the condition that $\bar \gamma+\alpha\sigma<1$,
\begin{align*}
\lim_{t\to\infty} \|\tilde{\w}^{t}\|_{b,\infty} \leq \frac{2\alpha\sigma\|\w^*\|_{b,\infty}+\alpha\|\nabla f(\w^*)\|_{b,\infty}}{1-(\bar \gamma+\alpha\sigma)}.
\end{align*}
From the definition of block maximum norm, \eqref{steadystate} is obtained.
\end{proof}

In iteration \eqref{mgd}, the transfer coefficient $m_{ij}^t$ between task $i$ and task $j$ is a scaler. In the following, we consider the element-wise feature similarities between task $i$ and task $j$. The transfer coefficient between task $i$ and task $j$ is assumed to be a diagonal matrix $P_{ij}\in\mathbb R^{d\times d}$ with its $k$-th diagonal element $P_{ij,k}$ being the transfer coefficient from the $k$-th element of $\bm w_j$ to the $k$-th element of $\bm w_i$. The MGD iteration in \eqref{mgd} is then becomes
\begin{align}\label{mgdp}
\bm w_i^{t+1} = \sum_{j=1}^T P_{ij}^t\bm w_j^t-\alpha\nabla f_i(\bm w_i^t),
\end{align}
where
\begin{align}
\nonumber &\sum_{j=1}^TP_{ij}^t = I_d,\\
&P_{ij,k}^t\geq 0,\ \forall i,j=1,...,T, k=1,...,d.
\end{align}
Following the same rescaling, 
\begin{align}\label{pbar}
\bar P_{ij}^t=
\begin{cases}
\frac{1}{\alpha\sigma}P_{ij}^t, & j\neq i,\\
I_d-\frac{1}{\alpha\sigma}\sum_{j\neq i}P_{ij}^t, & j=i,
\end{cases}
\end{align}
\eqref{mgdp} becomes
\begin{align}\label{mgdpbar}
\bm w_i^{t+1} = (1-\alpha\sigma)\bm w_i^t + \alpha\sigma\sum_{j=1}^T\bar P_{ij}^t\bm w_j^t-\alpha\nabla f_i(\bm w_i^t).
\end{align}
\begin{corollary}\label{coro1}
Under \eqref{mgdpbar} with the transfer coefficient $\bar P_{ij}^t$ satisfies 
\begin{align*}
\nonumber &\sum_{j=1}^T\bar P_{ij}^t = I_d,\\
&\bar P_{ij,k}^t\geq 0,\ \forall i,j=1,...,T, k=1,...,d,
\end{align*}
$\bm w_i^t$ is convergent if the following conditions are satisfied:
\begin{align*}
\sigma<\frac{\bar L_{f_i}}{T-1},\ \emph{for}\ T>1,\\
0<\alpha<\frac{2}{(T+1)\sigma+\bar L_{f_i}}.
\end{align*}

\end{corollary}
\begin{proof}
Let the $i,j$-th block element of $\bar{\mathcal P}^t\in\mathbb R^{dT\times dT}$ being $\bar P_{ij}^t\in\mathbb R^{d\times d}$. Following the similar procedure of the proof of Theorem~\ref{theo1}, we obtain
\begin{align}\label{coroite}
\nonumber&\|\tilde {\w}^{t+1}\|_{b,\infty} \\
\nonumber\leq &(\|(1-\alpha\sigma) I_{dT}-\alpha H^t\|_{b,\infty}+\alpha\sigma\|\bar{\mathcal P}^t\|_{b,\infty})\|\tilde {\w}^t\|_{b,\infty}\\
& +\alpha\|\sigma(\bar{\mathcal P}^t-I_{dT})\w^*-\nabla f(\w^*)\|_{b,\infty}.
\end{align}
Let $\bm x=[\bm x_i]_{\text{vec}}\in \mathbb R^{dT}$ being a block column vector with $\bm x_i\in\mathbb R^d$.
\begin{align*}
\|\bar {\mathcal P}^t \bm x\|_{b,\infty}=&\max_{i} \|\sum_{j=1}^T\bar P_{ij}^t \bm x_j\|\\
\leq &\max_i\sum_{j=1}^T\|\bar P_{ij}^t\|\|\bm x_j\|\\
\leq &(\max_i\sum_{j=1}^T\|\bar P_{ij}^t\|)\max_j\|\bm x_j\|.
\end{align*}
Recall that $\bar P_{ij}^t$ is a diagonal matrix and the elements therein are all no greater than 1, thus, $\sum_{j=1}^T\|\bar P_{ij}^t\|\leq T$. As a result
\begin{align*}
\|\bar {\mathcal P}^t \bm x\|_{b,\infty}\leq T\max_j\|\bm x_j\|.
\end{align*}
By the definition of matrix block maximum norm, we have
\begin{align*}
\|\bar{\mathcal P}^t\|_{b,\infty}\leq T.
\end{align*}
The condition to ensure convergence of the iteration in \eqref{coroite} becomes
\begin{align*}
\bar \gamma+\alpha\sigma T<1,
\end{align*}
which gives
\begin{align*}
\sigma<\frac{\bar L_{f_i}}{T-1},\ \text{for}\ T\neq 1,\\
0<\alpha<\frac{2}{(T+1)\sigma+\bar L_{f_i}}.
\end{align*}
\end{proof}

\subsection{Relation with Multi-Task Learning}
From the iteration in \eqref{barmite}, we have
\begin{align}\label{nashite}
\nonumber\bm w_i^{t+1} = (1-\alpha\sigma)\bm w_i^t + \alpha\sigma\sum_{j=1}^T\bar m_{ij}^t\bm w_j^t-\alpha\nabla f_i(\bm w_i^t)\\
=\bm w_i^t -\alpha(\sigma\sum_{j=1}^T\bar m_{ij}^t(\bm w_i^t-\bm w_j^t)+\nabla f_i(\bm w_i^t)).
\end{align}
If fix $\bar m_{ij}^t = \bar m_{ij}$ for all $t$, then, the last term in the brackets can be seen as the gradient of the following function
\begin{align*}
\bar f_i(\bm w_i, \bm w_{-i}) = f_i(\bm w_i)+\frac{1}{2}\sigma\sum_{j=1}^T \bar m_{ij}\|\bm w_i-\bm w_j\|^2,
\end{align*}
where $\bm w_{-i}$ denotes the collection of other tasks' variables, i.e., $\bm w_{-i}=[\bm w_1',...,\bm w_{i-1}',\bm w_{i+1}',...,\bm w_T']'$. Thus, the iteration in \eqref{nashite} with fixed $\bar m_{ij}$ can be seen as the gradient descent algorithm which solves the following Nash equilibrium problem
\begin{align}\label{nash}
\min_{\bm w_i}\ \bar f_i(\bm w_i, \bm w_{-i}),\ i=1,...,T.
\end{align}
In \eqref{nash}, each task's objective function is influenced by other tasks' decision variables. 
Since the objective function $\bar f_i(\bm w_i, \bm w_{-i})$ is continuous in all its arguments, strongly convex with respect to $\bm w_i$ for fixed $\bm w_{-i}$, and satisfies $\bar f_i(\bm w_i, \bm w_{-i})\to \infty$ as $\|\bm w_i\|\to \infty$ for fixed $\bm w_{-i}$, an Nash equilibrium exists \cite{basar1999dynamic}. Furthermore, as a result of strongly convexity, the gradient of $\bar f_i(\bm w_i, \bm w_{-i})$ with respect to $\bm w_i$ for fixed $\bm w_{-i}$ is strongly monotone. Thus, the Nash equilibrium for \eqref{nash} is unique \cite{facchinei2007finite}. Denote the Nash equilibrium of \eqref{nash} as $\bm w_i^o$, $i=\{1,...,T\}$. It is known that the Nash equilibrium satisfies the following condition \cite{basar1999dynamic}:
\begin{align*}
 \bm w_i^o = \text{argmin}_{\bm w_i} \bar f_i(\bm w_i, \bm w^o_{-i}),\ i=1,...,T,
\end{align*}
which implies 
\begin{align*}
\nabla f_i(\bm w_i^o)+\sigma\sum_{j=1}^T\bar m_{ij}(\bm w_i^o -\bm w_j^o) = 0, \ i=1,...,T.
\end{align*}
Write the conditions in \eqref{nashopt} in a concatenated form gives
\begin{align}\label{nashopt}
\nabla f(\w^o)+\sigma (I_{T}-\bar M)\otimes I_d\w^o = 0.
\end{align}

It has been pointed out in \cite{zhang2017survey} that the regularized multi-task learning algorithms which learn with task relations can be expressed as
\begin{align}\label{mtlt}
\min_{\bm w_i,\Sigma} \sum_{i=1}^T L_i(\bm w_i)+\frac{1}{2}\lambda\w^T(\Sigma^{-1}\otimes I_d)\w+g(\Sigma),
\end{align}
where $L_i$ is the training loss of task $i$, $\lambda$ is a positive regularization parameter, $\Sigma\in\mathbb R^{T\times T}$ models the task relations, and $g(\Sigma)$ denotes constraints on $\Sigma$. For comparison, we eliminate the constraints on $\Sigma$, consider the case that $\Sigma$ is fixed, and let $f(\w) = \sum_{i=1}^TL_{i}(\bm w_i)$. Denote the optimal solutions of problem \eqref{mtlt} as $\w^{g}$. The optimal solution satisfy the following condition,
\begin{align}\label{mtlopt}
\nabla f(\w^{g})+\frac{1}{2}\lambda(\Sigma^{-1}+(\Sigma^{-1})^T)\otimes I_d\w^{g}=0.
\end{align}

Comparing the optimality conditions \eqref{nashopt} and \eqref{mtlopt} for the Nash equilibrium problem \eqref{nash} and the multi-task learning problem \eqref{mtlt}, we find that if $\bar M$ can be set as
\begin{align}\label{condi}
\sigma (I_{T}-\bar M) = \frac{1}{2}\lambda(\Sigma^{-1}+(\Sigma^{-1})^T),
\end{align}
the optimal solution $\w^o$ will be the same as $\w^g$. The only limitation is that $\bar m_{ij}>0$, which can not cover the situation where there exists non-negative non-diagonal values in $\Sigma^{-1}$. Overall, the regularized multi-tasking learning problem with task relation learning can be solved by the MGD algorithm by setting the coefficients $\bar m_{ij}$ between task $i$ and task $j$ properly. In addition, using MGD, we can consider feature-feature relations between different tasks since we can use $\bar P_{ij}\in\mathbb R^{d\times d}$ as the transfer coefficient. Furthermore, in MGD, $\bar m_{ij}$ is not required to be equal to $\bar m_{ji}$. This relaxation allows asymmetric task relations in multi-task learning \cite{lee2016asymmetric}, which is hard to achieve by most multi-task learning methods since $\Sigma^{-1}+(\Sigma^{-1})^T$ is always symmetric in \eqref{mtlopt}.

Another category of regularized multi-task learning method is learning with feature relations \cite{zhang2017survey}. The objective function of this kind of method is
\begin{align}\label{mtlf}
\min_{\bm w_i,\Theta} \sum_{i=1}^T L_i(\bm w_i)+\frac{1}{2}\lambda\w^T(I_T\otimes \Theta^{-1})\w+g(\Theta),
\end{align}
where $\Theta\in\mathbb R^{d\times d}$ models the covariance between the features. The term $\w^T(I_T\otimes \Theta^{-1})\w$ can be decoupled as $\sum_{i=1}^T \bm w_i^T\Theta^{-1}\bm w_i$, which can be incorporated into $f_i(\bm w_i)$ for task $i$, and $\Theta$ can be learned using all the tasks' parameters during the optimization process using MGD.

\subsection{Incorporating Second-Order Label Correlations}
The transfer coefficients can be designed or learned by many different methods. In multi-label learning problems, the similarity between task $i$ and task $j$ can be modeled by the correlation between labels $\bm y^i$ and $\bm y^j$. In this paper, we use the cosine similarity to calculate the correlation matrix. The proposed MGD is summarized in Algorithm~\ref{alg}.
\begin{algorithm}[!htb] 
\caption{ The MGD Algorithm} 
\label{alg} 
\begin{algorithmic}[1] 
\REQUIRE ~~\\
The multi-label training set $\mathcal D_i = \{(X,\bm y^i)\}$, $i=1,...,T$\\
Hyperparameters in cost function $f_i$ for $i=1,...,T$ and $\sigma$\\
Step size $\alpha$, random initial values $\bm w_i^1$ for $i=1,...,T$\\
\ENSURE ~~\\
Model parameter $\bm w_i^*$ for $i=1,...,T$
\STATE Compute correlation matrix $C$ by cosine similarity on $Y=[\bm y^1,...,\bm y^T]$;
\STATE Normalize each row of $C$ to be row sum-to-one and set $\bar m_{ij}$ equal the $i,j$-th element of $C$;
\STATE Compute $m_{ij}$ by rescaling $\bar m_{ij}$ according to \eqref{mbar}; 
\REPEAT 
\STATE {Calculate the gradient $\nabla f_i(\bm w_i^t)$, $i=1,...,T$;}
\STATE {Update $\bm w_i$ according to \eqref{mgd};}
\UNTIL{Stop criterion reached;}
\RETURN $\bm w_i^* = \bm w_i^{t+1}$, $i=1,...,T$.
\end{algorithmic}
\end{algorithm}
After learning the model parameter $\bm w_i^*$, we can predict the label $y_t^i$ for a test instance $\bm x_t$ by the corresponding prediction function associated with the cost function, and the final predicted label vector is $[y_t^1,...,y_t^T]$.

\subsection{Complexity Analysis}
We mainly analyze the complexity of the iteration parts listed in Algorithm~\ref{alg}. In each iteration, the gradient calculation leads to a complexity of $\mathcal O(g(d)nT)$, where $g(d)$ is the complexity of calculating the gradient w.r.t. the dimension $d$, which is determined by the actual cost function, and the update of the model parameter according to \eqref{mgd} needs $\mathcal O(dT^2)$. Therefore, the overall complexity of the MGD algorithms is of order $\mathcal O(t(ng(d)T+dT^2))$, where $t$ is the iteration times.

\section{Experiments}
In this section, we extensively compared the proposed MGD algorithm with related approaches on real-world datasets. For the proposed MGD algorithm, we reformulate the multi-label learning problem, which can be decomposed into a set of binary classification tasks. For each of the classification tasks, we use the function of 2-norm regularized logistic regression. Thus, for any task $i$, the following objective function is optimized by the proposed algorithm,
\begin{align*}
\min_{\bm w_i} f_i(\bm w_i) =& -\frac{1}{n}\sum_{j=1}^n(y^i_j\log h(z^i_j)\\
&+(1-y^i_j)\log (1-h(z^i_j)))+\frac{1}{2}\rho\|\bm w_{i,-1}\|^2,
\end{align*}
where $h(z^i_j) = P(y^i_j=1|x_j) = \frac{1}{1+e^{-z^i_j}}$, $z^i_j = [1\ \bm x_j^T]\bm w_i$, $\bm w_i\in\mathbb R^{p+1}$ is the model parameter, $\bm w_{i,-1}\in\mathbb R^p$ is the remaining elements in $\bm w_i$ except the first element, and $\rho$ is the regularization parameter. 
The gradient of $f_i(\bm w_i)$ over $\bm w_i$ is
\begin{align*}
\nabla f_i =\frac{1}{n}\sum_{j=1}^n(h(z^i_j)-y^i_j)[1\ \bm x_j^T]^T+\rho\begin{bmatrix}
                                                                    0\\
                                                                    \bm w_{i,-1}
                                                                    \end{bmatrix}.
\end{align*}
Let
\begin{align*}
 X = 
 \begin{bmatrix}
 1 & \bm x_1^T \\
 \vdots&\vdots\\
 1 & \bm x_n^T
 \end{bmatrix}, 
 \bm y^i = 
 \begin{bmatrix}
 y_1^i\\
 \vdots\\
 y_n^i
 \end{bmatrix}.
 \end{align*}
The MGD iteration is
\begin{align*}
\bm w_i^{t+1} = &\sum_{j}^T m_{ij}^t\bm w_j^t-\frac{\alpha}{n} X^T(g(X\bm w_i^t)-\bm y^i)-\alpha\rho\begin{bmatrix}
                                                                    0\\
                                                                    \bm w_{i,-1}
                                                                    \end{bmatrix},
\end{align*}
where $g(X\bm w_i^t) = [\frac{1}{1+e^{-[1\ \bm x_1^T]\bm w_i^t}},...,\frac{1}{1+e^{-[1\ \bm x_n^T]\bm w_i^t}}]'$.

The $i$-th label prediction for an instant $\bm x_t$ is predicted 1 if $h(z_t^i)\geq \eta$ and 0 otherwise, where $\eta$ is the threshold. In the experiment, $\eta$ is chosen from $\{0.1,0.2,0.3\}$.
\subsection{Experimental Setup}
\subsubsection{Datasets}
We conduct the multi-label classification on six benchmark multi-label datasets, including regular-scale datasets: emotions, genbase, cal500, and enron; and large-scale datasets: corel5k and bibtex. The details of the datasets are summarized in Table~\ref{dataset}, where $|S|$, $\text{dim}(S)$, L$(S)$, $\text{Card}(S)$, and $\text{Dom}(S)$ represent the number of examples, the number of features, the number of class labels, the average number of labels per example, and feature type of dataset $S$, respectively. The datasets are downloaded from the website of Mulan \footnote{http://mulan.sourceforge.net/datasets-mlc.html} \cite{tsoumakas2009mining}.
\begin{table}[!htb]
    \centering
    \normalsize\caption{Characteristics of the tested multi-label datasets. $|S|$ represents the number of examples, $\text{dim}(S)$ represents the number of features, L$(S)$ represents the number of class labels, $\text{Card}(S)$ represents the average number of labels per example, and $\text{Dom}(S)$ represents feature type of dataset $S$.}
    \label{dataset}   
    \normalsize\centering\begin{tabular*}{\hsize}{@{}@{\extracolsep{\fill}}llllll@{}}
        \hline\hline
        Dataset  & $|S|$& $\text{dim}(S)$ & L$(S)$ & $\text{Card}(S)$ & Dom$(S)$ \\
        \noalign{\smallskip}\hline\noalign{\smallskip}
        emotions & 593  & 72              & 6      & 1.869            & music\\
        genbase  & 662  & 1186            & 27     & 1.252            & biology\\
        cal500   & 502  & 68              & 174    & 26.044           & music\\
        enron    & 1702 & 1001            & 53     & 3.378            & text\\
        corel5k  & 5000 & 499             & 374    & 3.522            & images\\
        bibtex   & 7395 & 1836            & 159    & 2.402            & text\\
        \hline\hline
    \end{tabular*}
\end{table}

\subsubsection{Evaluation Metrics}
Five widely used evaluation metrics are employed to evaluate the performance, including Average precision, Macro-averaging F1, Micro-averaging F1, Coverage score, and Ranking loss. Concrete metric definitions can be found in \cite{zhang2014review}. Note that for the comparison purpose, the coverage score is normalized by the number of labels. For Average precision, Macro averaging F1, and Micro averaging F1, the larger the values the better the performance. For the other two metrics, the smaller the values the better the performance.

\subsubsection{Comparing Algorithms}
We compare our proposed method MGD with three classical algorithms including BR \cite{boutell2004learning}, RAkEL \cite{tsoumakas2010random}, ECC \cite{read2011classifier}, and two state-of-the-art multi-label learning algorithms LIFT \cite{zhang2014lift} and LLSF-DL \cite{huang2016learning}. 

In the experiments, we used the source codes provided by the authors for implementation. BR, ECC, and RAkEL are implemented under the Mulan multi-label learning package \cite{tsoumakas2009mining} using the logistic regression model as the base classifier. Parameters suggested in the corresponding literatures are used, i.e., RAkEL: ensemble size 2$T$ with $k=3$; ECC: ensemble size 30; LIFT: the ratio parameter $r$ is tuned in \{0.1,0.2,...,0.5\}; LLSF-DL: $\alpha$, $\beta$, $\gamma$ are searched in $\{4^{-5},4^{-4},...,4^5\}$, and $\rho$ is searched in $\{0.1,1,10\}$. 
For the proposed approach MGD, $\alpha$ is set to 0.02, $\rho$ is chosen from $\{0.1,0.2,...,1\}$, and $\sigma$ is chosen from $\{0,0.05,0.1,0.15,...,0.3\}$.

\subsection{Experimental Results}
We run the algorithms 5 times on five sets of randomly partitioned training (80 percent) and
testing (20 percent) data, the mean metric values with standard deviations are recorded in Table~\ref{resultsP}. The best performance is shown in boldface, $\uparrow$ indicates the larger the better, and $\downarrow$ indicates the smaller the better. From the results, we can see that MGD outperforms other comparing algorithms in most cases. Specifically, MGD ranks first in 86.7\% cases. Compared with the existing algorithms, MGD introduced a new approach to incorporate label correlations, which is easy to implement and has low complexity. The results demonstrate the effectiveness of the proposed approach in improving the learning performance. 

\begin{table*}[!htb]
\caption{Prediction performance (mean $\pm$ std. deviation) on the tested datasets. Best performance is shown in boldface.}\label{resultsP}
\centering
\begin{threeparttable}[t]
\centering
\begin{tabular}{l|cccccc}
\hline\hline
Dataset  & emotions                & genbase                 & cal500                       & enron                      &corel5k                          &bibtex               \\ 
\hline\hline
 Algorithm & \multicolumn{6}{c}{Average precision $\uparrow$}\\
                                    \hline
   
                                MGD       &\textbf{0.815 $\pm$ 0.014} &\textbf{0.994 $\pm$ 0.006} &\textbf{0.516 $\pm$ 0.012}&\textbf{0.704 $\pm$ 0.016}  & \textbf{0.326 $\pm$ 0.010} & \textbf{0.596 $\pm$ 0.007}    \\ \cline{1-1}
                                BR        & 0.783 $\pm$ 0.027         & 0.985 $\pm$ 0.009       & 0.323 $\pm$ 0.009          & 0.384 $\pm$ 0.009          & 0.132 $\pm$ 0.004          & 0.199 $\pm$ 0.009     \\ \cline{1-1}
                                RAkEL     & 0.782 $\pm$ 0.030         & 0.575 $\pm$ 0.032       & 0.143 $\pm$ 0.003          & 0.168 $\pm$ 0.005          & 0.119 $\pm$ 0.004          & 0.323 $\pm$ 0.008     \\ \cline{1-1}
                                ECC       & 0.774 $\pm$ 0.030         & 0.992 $\pm$ 0.005       & 0.437 $\pm$ 0.007          & 0.554 $\pm$ 0.014          & 0.232 $\pm$ 0.006          & 0.441 $\pm$ 0.011      \\ \cline{1-1}
                                LIFT      & 0.734 $\pm$ 0.013         & 0.535 $\pm$ 0.031       & 0.502 $\pm$ 0.009          & 0.696 $\pm$ 0.011          & 0.289 $\pm$ 0.005          & 0.566 $\pm$ 0.010  \\ \cline{1-1}
                                LLSF-DL   & 0.710 $\pm$ 0.018         & 0.619 $\pm$ 0.053       & 0.470 $\pm$ 0.023          & 0.635 $\pm$ 0.018          & 0.271 $\pm$ 0.008          & 0.593 $\pm$ 0.004  \\ 
                                \hline\hline
 Algorithm& \multicolumn{6}{c}{Macro-averaging F1 $\uparrow$}\\
 \hline
                                MGD       &\textbf{0.668 $\pm$ 0.015} & 0.652 $\pm$ 0.076       &\textbf{0.191 $\pm$ 0.003}  & 0.226 $\pm$ 0.016          & 0.051 $\pm$ 0.003          & \textbf{0.336 $\pm$ 0.003}       \\ \cline{1-1}
                                BR        & 0.619 $\pm$ 0.037         &\textbf{0.915 $\pm$ 0.036}& 0.155 $\pm$ 0.007         & 0.206 $\pm$ 0.021          & 0.148 $\pm$ 0.007          & 0.136 $\pm$ 0.004       \\ \cline{1-1}
                                RAkEL     & 0.629 $\pm$ 0.034         & 0.661 $\pm$ 0.021       & 0.060 $\pm$ 0.011          & 0.112 $\pm$ 0.012          & 0.162 $\pm$ 0.012          & 0.202 $\pm$ 0.008      \\ \cline{1-1}
                                ECC       & 0.622 $\pm$ 0.033         & 0.904 $\pm$ 0.042       & 0.158 $\pm$ 0.010          &\textbf{0.252 $\pm$ 0.017}  & \textbf{0.208 $\pm$ 0.014} & 0.256 $\pm$ 0.009       \\ \cline{1-1}
                                LIFT      & 0.432 $\pm$ 0.017         & 0.026 $\pm$ 0.003       & 0.045 $\pm$ 0.002          & 0.141 $\pm$ 0.011          & 0.024 $\pm$ 0.001          & 0.218 $\pm$ 0.014 \\ \cline{1-1}
                                LLSF-DL   & 0.123 $\pm$ 0.024         & 0.006 $\pm$ 0.003       & 0.143 $\pm$ 0.006          & 0.195 $\pm$ 0.017          & 0.040 $\pm$ 0.003          & 0.210 $\pm$ 0.006  \\ 
                                \hline\hline
Algorithm& \multicolumn{6}{c}{Micro-averaging F1 $\uparrow$}\\
 \hline  
MGD       &\textbf{0.679 $\pm$ 0.014} & 0.966 $\pm$ 0.023       &\textbf{0.481 $\pm$ 0.010}  &\textbf{0.602 $\pm$ 0.014}     & \textbf{0.291 $\pm$ 0.009} & \textbf{0.413 $\pm$ 0.002}     \\ \cline{1-1}
                                BR        & 0.632 $\pm$ 0.035       &\textbf{0.974 $\pm$ 0.010} & 0.331 $\pm$ 0.005          & 0.356 $\pm$ 0.013          & 0.120 $\pm$ 0.003          & 0.145 $\pm$ 0.006       \\ \cline{1-1}
                                RAkEL     & 0.644 $\pm$ 0.035       & 0.740 $\pm$ 0.038         & 0.073 $\pm$ 0.004          & 0.182 $\pm$ 0.008          & 0.132 $\pm$ 0.003          & 0.211 $\pm$ 0.007     \\ \cline{1-1}
                                ECC       & 0.636 $\pm$ 0.031       & 0.926 $\pm$ 0.015         & 0.357 $\pm$ 0.008          & 0.457 $\pm$ 0.014          & 0.091 $\pm$ 0.007          & 0.352 $\pm$ 0.009  \\ \cline{1-1}
                                LIFT      & 0.506 $\pm$ 0.012       & 0.219 $\pm$ 0.025         & 0.316 $\pm$ 0.004          & 0.560 $\pm$ 0.012          & 0.077 $\pm$ 0.004          & 0.378 $\pm$ 0.009      \\ \cline{1-1}
                                LLSF-DL   & 0.199 $\pm$ 0.015       & 0.038 $\pm$ 0.022         & 0.459 $\pm$ 0.014          & 0.548 $\pm$ 0.019          & 0.249 $\pm$ 0.015          & 0.397 $\pm$ 0.010     \\ 
                                \hline\hline
                      \hline\hline 
Algorithm& \multicolumn{6}{c}{Coverage score $\downarrow$}\\
 \hline                                
MGD                                       &\textbf{0.291 $\pm$ 0.016} &\textbf{0.009 $\pm$ 0.004} & 0.740 $\pm$ 0.005        &\textbf{0.218 $\pm$ 0.015}  & \textbf{0.292 $\pm$ 0.003} & \textbf{0.105 $\pm$ 0.003}      \\ \cline{1-1}
                                BR        & 0.314 $\pm$ 0.029       & 0.016 $\pm$ 0.008       & 0.803 $\pm$ 0.007            & 0.259 $\pm$ 0.012          & 0.704 $\pm$ 0.007          & 0.426 $\pm$ 0.013          \\ \cline{1-1}
                                RAkEL     & 0.332 $\pm$ 0.024       & 0.370 $\pm$ 0.022       & 0.983 $\pm$ 0.002            & 0.819 $\pm$ 0.006          & 0.864 $\pm$ 0.005          & 0.366 $\pm$ 0.012            \\ \cline{1-1}
                                ECC       & 0.327 $\pm$ 0.036       & 0.013 $\pm$ 0.005       & 0.794 $\pm$ 0.007            & 0.292 $\pm$ 0.011          & 0.433 $\pm$ 0.007          & 0.236 $\pm$ 0.011              \\ \cline{1-1}
                                LIFT      & 0.358 $\pm$ 0.007       & 0.161 $\pm$ 0.829       & 0.748 $\pm$ 0.012            & 0.224 $\pm$ 0.011          & \textbf{0.292 $\pm$ 0.005} & 0.137 $\pm$ 0.006     \\ \cline{1-1}
                                LLSF-DL   & 0.373 $\pm$ 0.023       & 0.175 $\pm$ 0.036       &\textbf{0.733 $\pm$ 0.007}    & 0.336 $\pm$ 0.013          & 0.486 $\pm$ 0.012          & 0.185 $\pm$ 0.008    \\  
                                \hline\hline
Algorithm& \multicolumn{6}{c}{Ranking loss $\downarrow$}\\
 \hline                                
MGD       &\textbf{0.152 $\pm$ 0.009} &\textbf{0.001 $\pm$ 0.001} &\textbf{0.176 $\pm$ 0.003} &\textbf{0.075 $\pm$ 0.007}   & 0.136 $\pm$ 0.003           & \textbf{0.056 $\pm$ 0.001}     \\ \cline{1-1}
                                BR        & 0.180 $\pm$ 0.027       & 0.005 $\pm$ 0.004       & 0.243 $\pm$ 0.005            & 0.308 $\pm$ 0.010          & 0.368 $\pm$ 0.007          & 0.274 $\pm$ 0.008     \\ \cline{1-1}
                                RAkEL     & 0.194 $\pm$ 0.027       & 0.361 $\pm$ 0.024       & 0.604 $\pm$ 0.004            & 0.587 $\pm$ 0.005          & 0.573 $\pm$ 0.007          & 0.222 $\pm$ 0.008             \\ \cline{1-1}
                                ECC       & 0.194 $\pm$ 0.036       & 0.002 $\pm$ 0.002       & 0.222 $\pm$ 0.003            & 0.119 $\pm$ 0.005          & 0.192 $\pm$ 0.002          & 0.134 $\pm$ 0.008          \\ \cline{1-1}
                                LIFT      & 0.233 $\pm$ 0.007       & 0.138 $\pm$ 0.023       & 0.181 $\pm$ 0.012            & 0.077 $\pm$ 0.006          & \textbf{0.123 $\pm$ 0.002} & 0.075 $\pm$ 0.004     \\ \cline{1-1}
                                LLSF-DL   & 0.254 $\pm$ 0.025       & 0.161 $\pm$ 0.033       & 0.198 $\pm$ 0.010            & 0.130 $\pm$ 0.007          & 0.238 $\pm$ 0.005          & 0.097 $\pm$ 0.004      \\ 
                                \hline\hline
\end{tabular}
   \end{threeparttable}
\end{table*}

\subsubsection{Convergence Analysis} 
Compared to single gradient descent, the transfer in MGD also helps to accelerate the convergence. To see this, the iterations of the total loss, i.e., $\sum_{i=1}^T f_i(\bm w_i)$, are plotted in the first row of Figure~\ref{syn} for three datasets. It can be seen that the MGD converges faster than single gradient descent, especially at early iterations. The iterations of the average precision score are also plotted in the second row of Figure~\ref{syn}. It can be seen that for limited iteration times, the score under MGD is much better than single gradient descent. 

\begin{figure}[!htb]
\centering
\includegraphics[width=\linewidth]{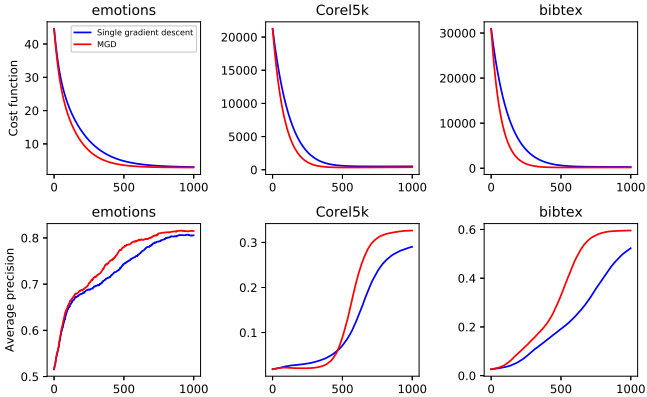}
\caption{Convergence test.}\label{syn}
\end{figure}
 
\subsubsection{Sensitivity Analysis}
We investigate the sensitivity of MGD with respect to the two hyperparameters $\rho$ and $\sigma$, which control the norm 2 regularization strength in the logistic regression and the transfer strength. Due to space limit, we only report the results on the emotions dataset using the average precision score. Figure~\ref{sen} shows how the average precision score varies with respect to $\rho$ and $\sigma$. Figure~\ref{sen}(b) and (c) are obtained by keeping the other parameter fixed at its best setting. It can be seen that both $\rho$ and $\sigma$ influence the performance. While, under a relatively wide range of parameters combinations, the score does not vary too much.  
\begin{figure}[!htb]
\centering
\includegraphics[width=\linewidth]{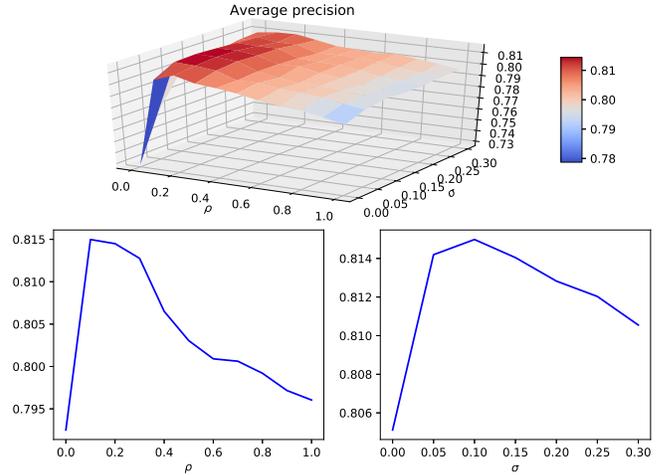}
\caption{Sensitivity analysis on the emotions dataset.}\label{sen}
\end{figure}

\section{Conclusion}
 In this paper, we propose the MGD algorithm for multi-label learning. Different from the state-of-the-art, MGD treats multi-label learning as multiple independent learning tasks, and multi-label learning can be significantly improved only via transferring correlated model parameter values during the learning of the independent labels. The convergence property of the learning mechanism has been theoretically proven. MGD is easy to implement, has less requirement on the training model, can achieve seamless asymmetric transformation such that negative transfer is mitigated, and can benefit from parallel
computing when the number of tasks is large. The proposed algorithm has been tested on multi-label learning datasets and has been compared with both classical and the state-of-the-art approaches. The competitive experimental results validate the effectiveness of MGD.

\bibliographystyle{ieeetr}
\bibliography{myref}

\end{document}